\documentclass[runningheads]{llncs}

\usepackage{amsmath}
\usepackage{amssymb}  
\usepackage{latexsym} 
\usepackage{multirow}
\usepackage{proof}

\usepackage{amssymb,amsfonts,stmaryrd} 
\usepackage{amsmath} 
\usepackage{latexsym} 
\usepackage{epsf}  
\usepackage{proof} 
\usepackage{lscape} 
\usepackage{verbatim}    
\usepackage{times}       
\usepackage{pdfsync}  
\usepackage{color}  
\usepackage{graphicx}
\usepackage[PostScript=dvips]{diagrams}

\newarrow{inc}C--->

\newcommand{\LL}{\mathcal{L}}
\newcommand{\FF}{\mathcal{F}}
\newcommand{\CC}{\mathcal{C}}
\newcommand{\Cond}{\mathsf{Cond}}
\newcommand{\Var}{\mathsf{Var}}
\newcommand{\op}{\mathsf{op}}
\newcommand{\ie}{\textit{i.e.}~}
\newcommand{\lrarr}{\longrightarrow}
\newcommand{\rarr}{\rightarrow}
\newcommand{\Set}{\mathbf{Set}}
\newcommand{\plusminus}{\pm}
\newcommand{\IFF}{\; \Longleftrightarrow \;}
\newcommand{\Image}{\mathrm{Im}}

\begin{document}

\title{Semantic Unification}
\subtitle{A sheaf theoretic approach to natural language}

\author{Samson Abramsky and Mehrnoosh Sadrzadeh}

\authorrunning{Abramsky and Sadrzadeh}
	
\institute{Department of Computer Science, University of Oxford \\ School of Electronic Engineering and Computer Science, Queen Mary University of London\\ \ \\ 
{\tt samson.abramsky@cs.ox.ac.uk}\\ {\tt mehrnoosh.sadrzadeh@eecs.qmul.ac.uk}}

\maketitle

\begin{abstract}
Language is contextual and sheaf theory provides a high level mathematical framework to model contextuality.   We show how sheaf theory can model  the contextual nature of natural language and how  gluing can be used to provide a global semantics for a discourse by putting together the local logical semantics of each sentence within the discourse.  
We introduce a presheaf structure corresponding to a basic form of Discourse Representation Structures.
Within this setting, we formulate a notion of \emph{semantic unification} --- gluing meanings of parts of a discourse into a coherent whole --- as a form of sheaf-theoretic gluing. We illustrate this idea with a number of examples where it can used to represent resolutions of anaphoric references. We also discuss multivalued gluing, described using a distributions functor, which can be used to represent situations where multiple gluings are possible, and where we may need to rank them using quantitative measures.
\begin{center}
\textbf{Dedicated to Jim Lambek on the occasion of his 90th birthday.}
\end{center}
\end{abstract}

\section{Introduction}

Contextual models of language originate from the work of  Harris \cite{Harris}, who argued that  grammatical roles  of words can be learnt from their  linguistic contexts and  went on to test his theory on learning of morphemes.  Later, contextual models were also applied to learn meanings of words, based on the frequency of their occurrence in document copora; these  gave rise to  the distributional models of meaning \cite{Firth}.  Very recently, it was shown how one can combine the contextual models of meaning with formal models of grammars, and in particular pregroup grammars \cite{Lambek}, to obtain a compositional distributional semantics for natural language \cite{Coecke}. 

One can study the contextual nature of language from yet another perspective: the inter-relationships between the meanings of the properties expressed by a discourse.  This  allows for the local information expressed by individual properties to be glued to each other and to form a global semantics for the whole discourse. A  very representative example is  anaphora, where two language units that may occur in different, possibly far apart, sentences,  refer to one another and the meaning of the whole discourse cannot be determined without resolving what is referring to what.   Such phenomena occur  in plenty in everyday discourse, for example there are four anaphoric pronouns in the following extract   from a  BBC news article on 16th of May 2013:

\begin{quote}
One of Andoura's earliest memories is making soap with his grandmother. She was from a family of traditional Aleppo soap-makers and  handed down a closely-guarded recipe [$\cdots$] to him. Made from mixing oil from laurel trees [$\cdots$], it  uses no chemicals or other additives. 
\end{quote}

Anaphoric phenomena are also to blame for    the complications behind the infamous Donkey sentences `If a farmer owns a donkey, he beats it.' \cite{Geach}, where the usual Montgue-style language to logic translations fail  \cite{Montague} .  The first widely accepted framework that provided a formal solution   to these challenges was  Discourse Representation Theory (DRT) \cite{Kamp}. DRT was later turned compositional in the setting of  Dynamic Predicate Logic (DPL) \cite{Groendijk} and  extended to polarities to gain more expressive power, using actions of modules on monoids \cite{Visser}. However, the  problem with these solutions is the standard criticism made to Montague-style semantics: they treat meanings of words as vacuous relations over an indexical sets of variables.

The motivation behind this paper is two-fold. Firstly,  the first author has been working on sheaf theory   to reason about contextual phenomena as sheaves provide a natural way of gluing  the  information  of local sections to obtain  a  consistent global view of the whole situation. Originally introduced in algebraic topology,  recently they have  been used to model the contextual phenomena  in other fields such as  in quantum physics \cite{AB11,AbramskyHardy}  and  in database theory \cite{Abramsky}. Based on these and aware of the contextual nature of natural language,  the first author conjectured a possible application of sheaves to natural language. Independently, during a research visit to McGill in summer of 2009,  the second author was encouraged by Jim Lambek to look at DRT and DPL as alternatives  to Montague semantics and was  in particular  pointed to the capacities of  these dynamic structures   in providing a formal model of anaphoric reference in natural language.  In this paper, we bring  these two ideas together and show how a sheaf theoretic interpretation of DRT    allows us to unify   semantics of individual discourses via  gluing  and provide semantics for the whole discourse.  We  first use the sheaf theoretic interpretation of the existing machinery of  DRT and apply the setting to  resolve  \emph{constraint-based} anaphora.  We then show how the composition of the sheaf functor with a probability distribution functor  can be used to resolve the so called \emph{preferential} anaphora. In such cases,  more than one possible resolution is possible and frequencies of  occurrences  of discourse units  from document corpora and the principle of maximal entropy will help choose the most common solution.

\section{Sheaves}

We recall some preliminary definitions. A category ${\cal C}$ has objects and morphisms. We use $A, B, C$ to denote the objects and $f,g$ to denote the morphisms. Examples of morphisms are  $f \colon A \to B$ and $g \colon B \to C$.  Each object $A$ has an identity morphism, denoted by $Id_A \colon A \to A$.  The morphisms are closed under composition: given $f \colon A \to B$ and $g \colon B \to C$, there is a morphism $g \circ f \colon A \to C$. Composition is associative, with identity morphisms as units.

A covariant functor $F$ from a category ${\cal C}$ to a category ${\cal D}$ is a map $F \colon {\cal C} \to {\cal D}$, which  assigns to each object $A$ of ${\cal C}$ an object $F(A)$ of ${\cal D}$ and to each morphism  $f \colon A \to B$ of ${\cal C}$, a morphism $F(f) \colon F(A) \to F(B)$ of ${\cal D}$.  Moreover, it preserves the identities and the compositions of ${\cal C}$. That is, we have $F(Id_A) = Id_{F(A)}$ and $F(g \circ f) = F(g) \circ F(f)$. A contravariant functor reverses the order of morphisms, that is,  for $F \colon  {\cal C} \to {\cal D}$  a contravariant functor and $f \colon A \to B$ in ${\cal C}$, we have $F(f) \colon    F(B) \to F(A)$ in ${\cal D}$. 

Two examples of a category are the category  ${\bf Set}$ of sets and functions and the category ${\bf Pos}$ of posets and monotone maps.

A presheaf is a contravariant functor from a  small category $\CC$ to the category  of sets and functions, which means that it is a functor on   the \emph{opposite} (or dual) category of ${\cal C}$:
\[
F \colon {\cal C}^{op} \to {\bf Set}
\]
This functor assigns a set $F(A)$ to  each object $A$ of ${\cal C}$. To each morphism $f \colon A \to B$ of ${\cal C}$, it assigns a function $F(f) \colon F(B) \to F(A)$, usually referred to as a \emph{restriction map}. For each $b \in F(B)$, these are denoted as follows:
\[
F(f)(b)\ =  \  b \mid_f .
\]
 Since $F$ is a functor, it follows  that the restriction of an identity   is an identity, that is  for  $a \in A$ we have:
 \[
 F(Id_A)(a) \; = \;  a \mid_{Id_A} \; = \; a.
 \]
Moreover, the restriction of a composition $F(g \circ f) \colon F(C) \to F(A)$ is the  composition  of the  restrictions  $F(f) \circ F(g)$ for $f : A \to B$ and $g : B \to C$. That is for  $c \in C$ we have:
 \[
 F(g \circ f) (c) \; = \; c \mid_{g \circ f} \; = \; (c \mid_g) \mid_f .
 \]

The original setting for sheaf theory was topology, where the domain category $\CC$ is the poset of open subsets of a topological space $X$ under set inclusion. In this case, the arrows of $\CC$ are just the inclusion maps $i : U \rinc V$; and restriction along such a map can rewritten unambiguously by specifying the domain of $i$; thus for $U \subseteq V$ and $s \in F(V)$, we write $s |_{U}$.

The elements of $F(U)$ --- `the presheaf at stage $U$' --- are called  \emph{sections}. In the topological case, a presheaf is a sheaf iff it satisfies the following condition:
\begin{quote}
Suppose we are given a family of open subsets $U_i \subseteq U$ such that $\bigcup_i U_i = U$,  i.e. the family $\{U_i\}$  covers $U$. Suppose moreover that we are given   a family of sections $\{s_i \in F(U_i)\}$ that are compatible, that is for all $i,j$ the  two sections $s_i$ and $s_j$ agree on the intersection of two subsets $U_i$ and $U_j$, so that we have:
\[
s_i \mid_{U_i \cap U_j} = s_j \mid_{U_i \cap U_j} .
\]
Then  there exists a   unique section $s \in F(U)$ satisfying the following \emph{gluing condition}:
\[ s \mid_{U_i} = s_i \quad \mbox{for all $i$.}  \]
\end{quote}
Thus in a sheaf, we can always unify or glue compatible local information together in  a unique way to obtain a global section.

\section{Discourse Representation Theory and Anaphora}

We shall assume a background first-order language $\LL$ of relation symbols. There are no constants or function symbols in $\LL$.

 In Discourse Representation Theory (DRT), every discourse  $K$ is represented by a  Discourse Representation Structure (DRS). Such a structure is a pair of a set $U_K$ of discourse referents and a set $\Cond_K$ of  DRS conditions:
\[
(U_K, \Cond_K).
\]
Here we take $U_K$ to be simply a finite subset of $\Var$, the set of first-order variables. For the purpose of this paper, we can restrict this set to the set of referents. 

A \emph{basic DRS} is one in which the condition $\Cond_K$ is a set of first-order literals, \ie atomic sentences or their negations, over the set of variables $U_K$ and the relation symbols in $\LL$.

The full class of DRS\footnote{Note that we write DRS for the plural `Discourse representation Structures', rather than the clumsier `DRSs'.} is defined by mutual recursion over DRS and DRS conditions:
\begin{itemize}
\item If $X$ is a finite set of variables and $C$ is a finite set of DRS conditions, $(X, C)$ is a DRS.
\item A literal is a DRS condition.
\item If $K$ and $K'$ are DRS, then $\neg K$, $K \Rightarrow K'$ and $K \vee K'$ are DRS conditions.
\item If $K$ and $K'$ are DRS and $x$ is a variable, $K(\forall x)K'$ is a DRS condition.
\end{itemize}

Our discussion in the present paper will refer only to basic DRS. However, we believe that our approach extends to the general class of DRS.
Moreover, our semantic unification construction to some extent obviates the need for the extended forms of DRS conditions.

The structure corresponding to  a discourse followed  by another is obtained by a merge and a unification of the structures of each discourse. The merge of two DRS $K$ and $K'$ is defined as their disjoint union, defined below:
\[
K \oplus K' \ := \ (U_K \uplus U_{K'}, Cond_K \uplus  Cond_{K'})
\]
A merge is followed by a unification (also called matching or presupposition resolution), where certain referents are equated with each other. A unification is performed according to a set of accessibility constraints,  formalising various different ways linguistics deal with endophora resolution. These include constraints such as  as c-commanding, gender agreement, syntactic and semantic consistency  \cite{Mitkov}.

An example where anaphora  is fully resolved is `John owns a donkey. He beats it.'. The merge of the DRS  of each discourse of this example is:
\begin{align*}
\Big(\{x,y\}, \{John(x), Donkey(y), Own(x,y)\}\Big) \quad \oplus \quad \Big(\{\underline{v},\underline{w}\}, \{Beat({\underline{v}},\underline{w})\}\Big)&\\
= \Big(\{x,y,\underline{v},\underline{w}\}, \{John(x), Donkey(y), Own(x,y), Beat(\underline{v},\underline{w})\}\Big)
\end{align*}
Here, $\underline{v}$ can access $x$ and has agreement with it, hence  we  unify  them by equating $\underline{v} = x$. Also $\underline{w}$ can access $y$ and has agreement with it, hence we unify them as well by equating $\underline{w}=y$. As a  result  we obtain the   following DRS:
\[
\Big(\{x,y\}, \{John(x), Donkey(y), Own(x,y), Beat(x,y)\}\Big)
\]
An example where anaphora is partially resolved is  `John does not own a donkey. He beats it.', the DRS of which is as follows:
\[
\left(\{x\}, \{John(x), \neg \left(\{y\} ,\{Donkey(y), Own(x,y)\}\right)\}\right) \ \oplus \left(\{\underline{v}, \underline{w}\}, \{Beat(\underline{v},\underline{w})\}\right)
\]
 Here $\underline{v}$ can be equated with $x$, but $\underline{w}$ cannot be equated with $y$, since $y$ is in a nested DRS and cannot  be accessed by $\underline{w}$. Hence, anaphora is not fully resolved. 
 
 The  unification  step enables the DRT to model and resolve contextual language phenomena  by going from local to global conditions: it will make certain properties which held about  a subset of  referents,  hold about the whole set of referents.  This is exactly the local to global passage modelled by gluing in sheaves.

\section{From Sheaf Theory To  Anaphora}

\subsection{A presheaf for basic DRS}

We begin by defining a presheaf $\FF$ which represents basic DRS.

We define the category $\CC$ to have as objects pairs $(L, X)$ where
\begin{itemize}
\item $L \subseteq \LL$ is a finite vocabulary of relation symbols.
\item $X \subseteq \Var$ is a finite set of variables.
\end{itemize}
A morphism $\iota, f : (L, X) \lrarr (L', X')$ comprises
\begin{itemize}
\item An inclusion map $\iota : L \rinc L'$
\item A function $f : X \lrarr X'$.
\end{itemize}
Note that we can see such functions $f$ as performing several r\^oles:
\begin{itemize}
\item They can witness the inclusion of one set of variables in another.
\item They can describe relabellings of variables (this will become of use when quantifiers are introduced).
\item They can indicate where variables are being identified or merged; this happens when $f(x) = z = f(y)$.
\end{itemize}
We shall generally omit the inclusion map, simply writing morphisms in $\CC$ as $f : (L, X) \lrarr (L',X')$, where it is understood that $L \subseteq L'$.

The functor $\FF : \CC^{\op} \lrarr \Set$ is defined as follows:
\begin{itemize}
\item For each object $(L, X)$ of $\CC$, $\FF(L, X)$ will be the set of deductive closures of consistent  finite sets of literals over $X$ with respect to the vocabulary $L$.
\item For each morphism $f : (L, X) \rarr (L', Y)$, the restriction operation $\FF(f) : \FF(L', Y) \rarr \FF(L, X)$ is defined as follows. For $s \in \FF(Y)$ and $L$-literal $\plusminus A(\vec{x})$ over $X$:
\[ \FF(f)(s) \vdash \plusminus A(\vec{x}) \IFF s \vdash \plusminus A(f(\vec{x})) . \]
\end{itemize}

The functoriality of $\FF$ is easily verified. Note that deductive closures of finite sets of literals are finite up to logical equivalence. Asking for deductive closure is mathematically convenient, but could be finessed if necessary.

The idea is that a basic DRS $(X, s)$ with relation symbols in $L$ will correspond to $s \in \FF(L, X)$ in the presheaf --- in fact, to an object of the \emph{total category} associated to the presheaf \cite{MM92}.

\subsection{Gluing in $\FF$}

Strictly speaking, to develop sheaf notions in $\FF$, we should make use of a Grothendieck topology on $\CC$ \cite{MM92}.
In the present, rather short and preliminary account, we shall work with concrete definitions which will be adequate to our purposes here. 

We shall consider \emph{jointly surjective} families of maps $\{ f_i : (L_i, X_i) \lrarr (L, X) \}_{i \in I}$, \ie such that $\bigcup_i \Image f_i \; = \; X$; and also $L = \bigcup_i L_i$.

We can think of such families as specifying \emph{coverings} of $X$, allowing for relabellings and identifications.

We are given a family of elements (sections) $s_i \in \FF(L_i, X_i)$, $i \in I$. Each section $s_i$ is giving information local to $(L_i, X_i)$. A \emph{gluing} for this family, with respect to the cover $\{ f_i \}$, is an element $s \in \FF(L, X)$ --- a section which is \emph{global} to the whole of $(L, X)$ --- such that $\FF(f_i)(s) = s_i$ for all $i \in I$.

We shall interpret this construction as a form of \emph{semantic unification}. We are making models of the meanings of parts of a discourse, represented by the family $\{ s_i \}$, and then we glue them together to obtain a representation of the meaning of the whole discourse.
The gluing condition provides a general and mathematically robust way of specifying the adequacy of such a representation, with respect to the local pieces of information, and the identifications prescribed by the covering.

We have the following result for our presheaf $\FF$.

\begin{proposition}
\label{uniqprop}
Suppose we are given a cover $\{ f_i : (L_i, X_i) \lrarr (L, X) \}$.
% where the sets $L_i$ are pairwise disjoint.
If a gluing $s \in \FF(X)$ exists for a family $\{s_i \in \FF(L_i, X_i) \}_{i \in I}$ with respect to this cover, it is unique.
\end{proposition}
\begin{proof}
We define s as the deductive closure of 
\[ \{ \plusminus A(f_i(\vec{x})) \mid \plusminus A(\vec{x}) \in s_i, i \in I \} . \]
If $s$ is consistent and restricts  to $s_i$ along $f_i$ for each $i$, it is the unique gluing.
\end{proof}

\paragraph{Discussion and Example}
Note that, if the sets $L_i$ are \emph{pairwise disjoint}, the condition on restrictions will hold automatically if $s$ as constructed in the above proof is consistent.
To see how the gluing condition may otherwise fail, consider the following example.
We have $L_1 = \{ R, S \} = L_2 = L$, $X_1 = \{ x, u \}$, $X_2 = \{ y, v \}$, and $X = \{ z, w \}$. There is a  cover $f_i : (L_i, X_i) \lrarr (L, X)$, $i = 1, 2$, where $f_1 : x \mapsto z, u \mapsto w$, $f_2 : y \mapsto z, v \mapsto w$.
Then the sections $s_1 = \{ R(x), S(u) \}$, $s_2 = \{ S(y), R(v) \}$ do not have a gluing. The section $s$ constructed as in the proof of Proposition~\ref{uniqprop} will e.g. restrict along $f_1$ to $\{ R(x), S(x), R(u), S(u) \} \neq s_1$.
\subsection{Linguistic Applications}

We shall now discuss a number of examples in which semantic unification expressed as gluing of sections can be used to represent resolutions of anaphoric references.

In these examples, the r\^ole of \emph{merging} of discourse referents in DRT terms is represented by the specification of suitable cover; while the gluing represents merging at the semantic level, with the gluing condition expressing the semantic correctness of the merge.

Note that by Proposition~\ref{uniqprop}, the `intelligence' of the semantic unification operation is in the choice of cover; if the gluing exists relative to the specified cover, it is unique. Moreover, the vocabularies in the covers we shall consider will always be disjoint, so the only obstruction to existence is the consistency requirement.

\subsubsection*{Examples}

\begin{enumerate}
\item Consider firstly the discourse  `John sleeps. He snores.'
We have the local sections
\begin{eqnarray*}
 s_1 &=& \{ John(x), sleeps(x) \} \in \FF(\{ John, sleeps \}, \{ x\}), \\
  s_2 &=&  \{ snores(y) \} \in \FF(\{ snores \}, \{y\}) . 
  \end{eqnarray*}
To represent the merging of these discourse referents, we have the cover
\[ f_1 : \{ x \} \lrarr \{ z \} \longleftarrow \{ y \} . \]
A gluing of $s_1$ and $s_2$ with respect to this cover is given by
\[ s = \{ John(z), sleeps(z), snores(z) \} . \]
%This gluing is in fact unique in this case.

\item In intersentential  anaphora  both the anaphor and antecedent occur in one sentence.  An example is `John beats his donkey'. We can express the information conveyed in this sentence in three local sections:
\[ s_1 = \{ John(x) \}, \quad s_2 = \{ donkey(y) \}, \quad s_3 = \{owns(u,v), beats(u,v) \} \]
over $X_1 = \{x\}$, $X_2 = \{ y \}$ and $X_3 = \{ u, v \}$ respectively.

We consider the cover $f_i : X_i \lrarr \{ a, b \}$, $i=1,2,3$, given by
\[ f_1 : x \mapsto a, \quad f_2 : y \mapsto b, \quad f_3 : u \mapsto a, v \mapsto b . \]
The unique gluing $s \in \FF(\{ John, donkey, owns, beats \},\{ a, b \})$ with respect to this cover is
\[ s = \{ John(a), donkey(b), owns(a,b), beats(a, b) \} . \]

\item We illustrate the use of negative information, as expressed with negative literals, with the following example: `John owns a donkey. It is grey.'  The resolution method for this example is agreement;  we have to make it clear that `it' is a pronoun that does not refer to men. This is done using a negative literal. Ignoring for the moment the ownership predicate (which would have been dealt with in the same way as in the previous example), the local sections are as follows:
\[ s_1 = \{ John(x), Man(x) \}, \quad s_2 = \{ donkey(y), \neg Man(y) \}, \quad s_3 = \{ grey(z) \} \} . \]
Note that a cover which merged $x$ and $y$ would not have a gluing, since the consistency condition would be violated. However, using the cover 
\[ f_1 : x \mapsto a, \quad f_2 : y \mapsto b, \quad f_3 : z \mapsto  b , \]
we do have a gluing:
\[ s = \{ John(a), Man(a), donkey(b), \neg Man(b), grey(b) \} . \]

\item The following example illustrates the situation where we may have several plausible choices for covers with respect to which to perform gluing. 
Consider  `John put the cup on the plate. He broke it'. We can represent this by the following local sections
\[ s_1 = \{ John(x), Cup(y), Plate(z), PutOn(x, y, z) \}, \quad s_2 = \{ Broke(u, v) \} . \]
We can consider the cover given by the identity map on $\{ x, y, z \}$, and $u \mapsto x, v \mapsto y$; or alternatively, by $u \mapsto x, v \mapsto z$.

In the next section, we shall consider how such multiple possibilities can be ranked using quantitative information within our framework.
\end{enumerate}

\section{Probabilistic Anaphora}
Examples where  anaphora  cannot  be resolved by a constraint-based method  are plentiful, for instance in  `John has a brother. He is happy', or `John put a cd in the computer and copied it', or `John gave a donkey to Jim. James also gave him a dog', and so on. In such cases, although we are not sure which unit the anaphor refers to, we have some preferences. For instance in the first example,    it is more likely that `he' is referring to `John'. If instead we had `John has a brother. He is nice.', it would be more likely that `he' would be referring to `brother'. These considerations can be taken into account in a probabilistic setting. 

To model degrees of likelihood  of gluings, we compose our sheaf functor with a distribution functor as follows:
\[
{\CC}^{\op} \stackrel{\FF}{\lrarr} {\bf Set} \stackrel{D_R} {\lrarr} {\bf Set}
\]
The distribution functor is parameterized by a commutative semiring, that is a structure $(R, +, 0, \cdot, 1)$, where $(R, +, 0)$ and $(R, \cdot, 1)$ are commutative monoids, and we have the following distributivity property, for $x,y,z \in R$:
\[
x \cdot (y + z) = (x \cdot y) + (x \cdot z).
\]
Examples of semirings include the real numbers $\mathbb{R}$,  positive real numbers  $\mathbb{R}^+$,  and the booleans $\mathbf{ 2}$.  In the case of the reals and positive reals,  $+$ and $\cdot$ are addition and multiplication.  In the case of  booleans,  $+$ is disjunction and $\cdot$ is conjunction. 

Given a set $S$, we define $D_R(S)$ to be the set of functions $d : S \rarr R$ of finite support, such that
\[ \sum_{x \in S} d(x) = 1 . \]
For the  distribution functor over the booleans, $D(S)$ is the set of finite subsets of $S$,  hence $D$ becomes the finite powerset functor. 
To model probabilities,  we work with the distribution functor over $\mathbb{R}^+$. 
In this case, $D_R(S)$ is the set of finite-support probability measures over $S$.

The functorial action of $D_R$ is defined as follows. If $f : X \rarr Y$ is a function, then for $d \in D_R(X)$:
\[ D_R(f)(y) = \sum_{f(x)=y} d(x) . \]
This is the direct image in the boolean case, and the image measure in the probabilistic case.

\subsection{Multivalued Gluing}

If we now consider a family of probabilistic sections $\{ d_i \in D_R \FF(L_i,X_i) \}$, we can interpret the probability assigned by $d_i$ to each $s \in \FF(L_i, X_i)$ as saying how likely this condition is as the correct representation of the meaning of the part of the discourse the local section is representing.

When we consider this probabilistic case, there may be several possible gluings $d \in D_R \FF(L, X)$ of a given family with respect to a cover $\{ f_i : X_i \lrarr X \}$. 
We can use the principle of maximal entropy \cite{Jay54}, that is maximizing over 
$- \sum_{s \in \FF(L,X)} d(s) \log d(s)$,  to find out which of these sections is most probable. 
%The degrees of likelihood of $C_i$'s can be obtained by statistical processing of a document corpus or machine learning, by tailoring the methods suggested in  \cite{Dagan} and  \cite{Aone}. 
We can also use maximum entropy considerations to compare the likelihood of gluings arising from different coverings.

In the present paper, we shall study a more restricted situation, which captures a class of linguistically relevant examples.
We assume that, as before,  we have a family of deterministic sections $\{ s_i \in \FF(L_i, X_i) \}$, representing our preferred candidates to model the meanings of parts of a discourse. We now have a number of possible choices of cover, representing different possibilities for resolving anaphoric references.
Each of these choices $c$ will give rise to a different deterministic gluing $s_c \in \FF(L, X)$.
We furthermore assume that we have a distribution $d \in D_R \FF(L, X)$.
This distribution may for example have been obtained by statistical analysis of corpus data.

We can then use this distribution to rank the candidate gluings according to their degree of likelihood.
We shall consider an example to illustrate this procedure.

\subsection*{Example}
As an example consider the discourse:
\begin{quote}
John gave the bananas to the monkeys. They were ripe. They were cheeky.
\end{quote}

The meanings of the three sentences are represented by the following local sections:
\[ \begin{array}{lcl}
s_1 & = & \{John(x), Banana(y), Monkey(z), Gave(x,y,z)\}, \\
s_2 & = & \{ Ripe(u) \}, \\
s_3 & = & \{ Cheeky(v) \} . 
\end{array}
\]
There are four candidate coverings, represented by the following maps, which extend the identity on $\{ x, y, z \}$ in the following ways:
\[ c_1 : u \mapsto y, v \mapsto y \quad c_2 : u \mapsto y, v \mapsto z \quad c_3 : u \mapsto z, v \mapsto y \quad  c_4 : u \mapsto z, v \mapsto z . \]
These maps induce four candidate global sections, $t_1, \ldots , t_4$.
For example:
\[ t_1 = \{ John(x), Banana(y), Monkey(z), Gave(x,y,z), Ripe(y), Cheeky(y) \} . \]

We obtain  probability distributions for the coverings  using  the statistical method of \cite{Dagan}.  This method  induces  a grammatical relationship between the possible antecedents and the anaphors and obtains patterns for their  possible instantiations  by substituting  the  antecedents and anaphors into their assigned roles. It then counts how many times the lemmatised versions of the patterns  obtained from these substitutions   have occurred in a corpus. Each of these patterns correspond to a possible merging of referents.   The events we wish to assign probabilities  to are  certain combinations of  mergings of referents. The probability of each such event will be  the ratio of the sum of occurrences  of its mergings  to the total number of mergings in all events.  Remarkably,  these  events correspond to  the coverings of the sheaf model.

 In our example,  the sentences that contain the anaphors are predicative. Hence,  the induced relationship corresponding to their anaphor-antecedent pairs will be that of ``adjective-noun''. This  yields the  following four patterns, each corresponding to  a merging map, which is presented underneath it:
\begin{center}
\begin{tabular}{cccc}
 `ripe bananas', & `ripe monkeys', & `cheeky bananas',  & `cheeky monkeys'\\
 $u \mapsto y$  & $u \mapsto z$& $v \mapsto y$ & $v \mapsto z$
 \end{tabular}
 \end{center}
We query  the  \emph{British News corpus}  to obtain frequencies of the occurrences of the above patterns. This corpus is a collection of news stories from 2004 from each of the four major British newspapers: Guardian/Observer, Independent, Telegraph and Times. It contains 200 million words.  The  corresponding frequencies for these patterns are presented below:
\begin{center}
\begin{tabular}{lc}
 `ripe banana'  &  \quad 14\\
 `ripe monkey'  &   \quad 0\\
 `cheeky banana'  & \quad 0 \\
 `cheeky monkey'  & \quad 10
 \end{tabular}
\end{center}
The events are certain pairwaise combinations of the above,  namely  exactly the pairs whose mappings form a  covering. These coverings and  their probabilities are as follows: 
\begin{center}
\begin{tabular}{lcc}
{\sf \bf Event} &{\sf \bf Covering} & {\sf \bf Probability}\\
\hline
 `ripe banana' ,   `cheeky banana'   &\quad  $c_1 : u \mapsto y, v \mapsto y$ & \qquad 14/48 \\
 `ripe banana' ,   `cheeky monkey' &   \quad$c_2 : u \mapsto y, v \mapsto z$& \qquad (14+10)/ 48\\
 `ripe monkey' ,    `cheeky banana'  &  \quad$c_3 : u \mapsto z, v \mapsto y$& \qquad 0 \\
 `ripe monkey' ,  `cheeky monkey'  &     \quad$c_4 : u \mapsto z, v \mapsto z$& \qquad 10/48
\end{tabular}
\end{center}
These probabilities result in a  probability    distribution $d \in D_R \FF(L, X)$ for the gluings. The  distribution for the case of our example is as follows:

\begin{center}
\begin{tabular}{l|c|c}
$i$ &  $t_i$ & $d(t_i)$ \\
\hline
1&  $\{John(x), Banana(y), Monkey(z), Gave(x,y,z), Ripe(y), Cheeky(y)\}$ & 0.29 \\
2&  $\{John(x), Banana(y), Monkey(z), Gave(x,y,z), Ripe(y), Cheeky(z)\}$ & 0.5\\
3&  $\{John(x), Banana(y), Monkey(z), Gave(x,y,z), Ripe(z), Cheeky(y)\}$& 0\\
4&  $\{John(x), Banana(y), Monkey(z), Gave(x,y,z), Ripe(z), Cheeky(z)\}$& 0.205\\
\end{tabular}
\end{center} 
We can now select the candidate resolution $t_2$ as the most likely with respect to $d$. 

\section{Conclusions and Future Work}
We have shown how sheaves and gluing can be used to model the contextual nature of language, as represented by DRT and unification. We  provided examples of the constraint-based anaphora  resolution  in this setting and   showed how a move to preference-based cases is possible  by composing the sheaf functor with a distribution functor, which  enables one to    choose between a number of possible resolutions. 

There are a number of interesting directions for future work:
\begin{itemize}
\item We aim to extend our sheaf-theoretic  treatment of DRT to  its logical operations. The model-theoretic semantics of DRS has an intuitionistic flavour, and we aim to develop a sheaf-theoretic form of this semantics.
\item The complexity of anaphora resolution has been a concern for linguistics; in our setting we can approach this matter by characterizing the complexity of finding a gluing. The recent work in \cite{AGK13} seems relevant here.
\item We would like to experiment with different statistical ways of learning the distributions of DRS conditions on large scale corpora and real linguistic tasks, in the style of \cite{Grefen}, and how this can be fed back into the sheaf-theoretic approach, in order to combine the strengths of structural and statistical methods in natural language semantics.
\end{itemize}

\end{document}